\documentclass[journal]{IEEEtran}

\usepackage{cite}
\usepackage{amsmath,amssymb,amsfonts,mathtools}
\usepackage{amsthm}
\usepackage[T1]{fontenc}
\usepackage{newtxtext}
\usepackage{graphicx}
\usepackage{booktabs}
\usepackage{xcolor}
\usepackage[hidelinks]{hyperref}

\graphicspath{{figures/}}

\theoremstyle{plain}
\newtheorem{theorem}{Theorem}
\newtheorem{lemma}[theorem]{Lemma}
\newtheorem{proposition}[theorem]{Proposition}
\newtheorem{corollary}[theorem]{Corollary}
\theoremstyle{definition}
\newtheorem{definition}[theorem]{Definition}
\theoremstyle{remark}
\newtheorem{remark}[theorem]{Remark}

\newcommand{\R}{\mathbb{R}}
\newcommand{\E}{\mathbb{E}}
\newcommand{\Var}{\operatorname{Var}}
\newcommand{\norm}[1]{\lVert #1 \rVert}
\newcommand{\abs}[1]{\lvert #1 \rvert}

\newcommand{\Proj}{\Pi}

\hypersetup{pdftitle={The Skin-Restricted Reinhard Transform},
            pdfauthor={Vijesh KP}}

\begin{document}

\title{The Skin-Restricted Reinhard Transform:\\
Uniqueness under a Lightness-Preserving Constraint}

\author{Vijesh~KP\\
{\normalsize\normalfont vijeshkpaei@gmail.com}}

\markboth{Skin-Restricted Reinhard Transform}%
{KP: The Skin-Restricted Reinhard Transform}

\maketitle

\begin{abstract}
\bfseries\boldmath
Catalog skin recolouring has to change pigment and leave shading alone.
The classical Reinhard map does not make that split: it rescales lightness by the ratio of standard deviations, and a flat reference swatch therefore flattens the limb.
This paper formalises the correction used in our pipeline, the \emph{skin-restricted Reinhard transform}.
It is the diagonal affine map in CIE Lab that translates lightness, matches the chromatic mean, and clamps the chromatic gain to $[0.72,1.18]$, with moments taken on the central $84\%$ of each channel.
A diagonal affine map has six real parameters.
The shading constraint forces the lightness gain to $+1$ and the lightness shift to the difference of means; one-dimensional quadratic optimal transport on each chromatic axis, followed by Euclidean projection $\Pi_{[0.72,1.18]}(\cdot)$ onto the gain interval, fixes the other four.
Inside that family the four conditions determine every parameter.
The content of the result is the forced lightness gain; it is not a uniqueness claim outside the diagonal affine class.
For Gaussian marginals the chromatic step is not merely the best affine map: it is the unrestricted Wasserstein-$2$ map.
The same formulae with trimmed moments remain optimal because a positive affine image commutes with quantile trimming.
On hands, arms, legs, and feet of nine photographs and three reference tones, the map keeps the lightness contrast ratio at $0.974\pm0.029$ with chromatic error $0.77$ CIE Lab units.
Reinhard matching, the linear Monge map, and histogram matching reach a smaller chromatic error only by cutting lightness contrast to about half.
\end{abstract}


\section{Introduction}
A catalog photograph is already lit.
The edit is to replace skin pigment with the pigment of a reference texture, without repainting folds, knuckles, veins, cloth, or accessories.
Those shading signals live in the lightness channel of CIE Lab.
Any transfer that treats lightness as one more histogram to be copied will spend its degrees of freedom on deleting them.

Reinhard \emph{et al.} align the mean and the standard deviation of each channel of a decorrelated opponent space~\cite{reinhard2001,ruderman1998}.
In Lab the same formula rescales lightness by $\sigma_{L,\mathrm{ref}}/\sigma_{L,\mathrm{src}}$.
Reference skin textures are close-ups.
Their lightness variance is far smaller than the variance of a lit arm or leg, so the gain is a flattening factor.
Full-covariance Monge maps, iterative distribution transfer, sliced optimal transport, and histogram matching copy still more of the reference law and therefore copy the same defect~\cite{pitie2007cvmp,pitie2007cviu,rabin2011,bonneel2015,neumann2005,pouli2011,faridul2016}.
Gradient-preserving transfer adds a penalty on that defect instead of removing the degree of freedom that causes it~\cite{xiao2009}.

The map studied here removes the degree of freedom.
Lightness may be translated and may not be rescaled.
Chroma may be matched by an affine map whose gain lies in a fixed interval, so a flat or noisy reference cannot wash the photograph out or push it to an unnatural saturation.
The support of the change is a skin matte, built from a human part parser rather than from a colour threshold.
This paper treats the colour map as the object of proof.
The claims are:

\begin{enumerate}
\item In the six-parameter family of diagonal affine Lab maps, the shading constraint and the chromatic moment constraint determine every parameter.
The system is square, and the lemmas show it has one solution.
The skin-restricted Reinhard transform is that solution.
The counting step is bookkeeping; the content is that the lightness gain is forced to $+1$.
\item For Gaussian chromatic marginals the same formula is the monotone optimal-transport map for the quadratic Wasserstein cost, and the gain clamp is the metric projection of that map onto the allowed interval.
\item Trimmed means and trimmed variances, the estimator used in the implementation, inherit the same uniqueness because positive affine maps commute with quantile windows.
\item Classical Reinhard, histogram matching, and any transport that reproduces a reference lightness law of a different variance are infeasible for the shading constraint, however small their colour error.
\end{enumerate}

The experimental section scores the map on limb skin, and reports the same statistics on the face band of that catalog set.
Appendix~\ref{app:visual} is a separate qualitative plate.
The gain interval and the trim window are operating points.
A sensitivity table shows what they change.

\section{Related Work}
Reinhard \emph{et al.} transfer colour by matching per-channel moments in the $l\alpha\beta$ space of Ruderman \emph{et al.}~\cite{ruderman1998,reinhard2001}.
Written in Lab, with $c\in\{L,a,b\}$,
\begin{equation}
\label{eq:reinhard}
c'
=
\frac{\sigma_{c,\mathrm{ref}}}{\sigma_{c,\mathrm{src}}}
\bigl(c-\mu_{c,\mathrm{src}}\bigr)
+\mu_{c,\mathrm{ref}}.
\end{equation}
The survey of Faridul \emph{et al.} records the artefacts of this construction: flattened contrast, colour spill, and loss of the source gradient~\cite{faridul2016}.
The lightness factor in~\eqref{eq:reinhard} is exactly a uniform scaling of $\nabla L$.

Piti\'e, Kokaram, and Dahyot match one-dimensional marginals along random rotations, which is iterative distribution transfer~\cite{pitie2007cviu}.
The linear Monge--Kantorovich map is the closed-form quadratic-cost coupling of two Gaussians and matches the full covariance~\cite{pitie2007cvmp,villani2009}.
Sliced optimal transport advects colour along random lines~\cite{rabin2011,bonneel2015}.
Histogram matching is the monotone rearrangement of each channel~\cite{neumann2005}, and Pouli and Reinhard apply it progressively across scales~\cite{pouli2011}.
Each of these is optimal for a cost on colour samples.
None of them contains the constraint $\Var(L')=\Var(L)$.
Xiao and Ma keep source gradients by an explicit penalty in a correlated colour space~\cite{xiao2009}.
The constraint used here is harder and cheaper: the lightness Jacobian entry is fixed at $1$, and the map stays diagonal.
A companion operator injects form shadow on a composited face by a channel-uniform multiply in linear RGB, which darkens and does not shift chromaticity~\cite{kp2026shadow}.
It does not transfer pigment.

Skin detectors in YCbCr and HSV decide where a colour is allowed to change~\cite{hsu2002,phung2005,jones2002,kakumanu2007}.
They do not decide which axis is pigment and which is illumination, and they fire on cloth whose chroma overlaps skin.
A part parser supplies a different support.
We use the $29$-class Sapiens2 segmentation~\cite{khirodkar2024}.
Makeup transfer~\cite{guo2009} and deep photographic style transfer~\cite{luan2017} move appearance that includes texture.
They solve a wider problem than a catalog grade.
The guided filter places the matte along lightness edges~\cite{he2013}.
It is not part of the uniqueness argument; the argument is about the colour map inside the matte.

\section{Skin Support}
\label{sec:support}
Sapiens2 assigns every pixel one of $29$ labels~\cite{khirodkar2024}.
The candidate skin set $\Omega_S$ is the union of face, torso, upper and lower arms, hands, upper and lower legs, and feet.
Hair and eyeglasses are removed after a one-pixel dilation.
Shoes, socks, and clothing labels are forbidden to every later growth step.
Connected components smaller than $\max(250,\,0.02 A_{\max})$ pixels are dropped, where $A_{\max}$ is the area of the largest component.

A white trouser or a black sleeve is sometimes labelled as a limb.
Let $m_F$ be the coordinate-wise Lab median on the face label and let $C_F$ be the median chroma $\sqrt{a^2+b^2}$ there.
A limb component $B$ is rejected when
\begin{equation}
\label{eq:cloth}
\begin{aligned}
&\norm{\operatorname{median}_{x\in B}\varphi(I)(x)-m_F}_2 > 22\\
&\text{and}\quad
\operatorname{median}_{x\in B} C(x) < \max(8,\,0.55\,C_F).
\end{aligned}
\end{equation}
The component must be both far from facial skin and nearly neutral.
Skin fails the second test.
A neutral garment passes it and fails the first, so the conjunction removes the garment and keeps the limb.
The distance $22$, the chroma factor $0.55$, and the floor $8$ are operating points of this gate.
No cloth-labeled set is available, so they are not claimed to be optimal.

Gaps are filled only where a local Lab colour agrees with nearby skin.
With $\bar c$ a $31\times 31$ average of channel $c$ inside the current mask, a non-clothing pixel is added when
\begin{equation}
\label{eq:local}
\sqrt{
0.45\Bigl(\frac{L-\bar L}{10}\Bigr)^2
+\Bigl(\frac{a-\bar a}{6}\Bigr)^2
+\Bigl(\frac{b-\bar b}{6}\Bigr)^2
}
\le 1.55
\end{equation}
and its chroma is at least $30\%$ of the median skin chroma.
The lightness term is down-weighted because a shadow is not a different surface.
The factors $0.45$, $10$, and $6$ are that design choice: a lightness change of a shadow stays inside the ball, and a jump in $a$ or $b$ of about nine units does not.
They were not selected by a sweep.
Neck pixels are accepted only in a box under the face, only where the parser said background, and only within a Lab ball of the facial median.

The matte $\alpha:\Omega\to[0,1]$ equals $1$ on body labels inside $\Omega_S$, equals $0$ on clothing and eyeglasses, and falls off by the cubic smoothstep $s(t)=t^2(3-2t)$ over a distance of about three pixels at the outer shell.
This cubic is the unique polynomial of degree three with $s(0)=0$, $s(1)=1$, and vanishing derivatives at both ends, so the falloff meets its plateaus with zero slope.
An interior Gaussian likelihood in standardised Lab, with lightness down-weighted by $0.4$ and scale $\sigma=2.35$, together with the guided filter of He \emph{et al.}~\cite{he2013}, is used only to judge ambiguous boundary pixels.
It does not reduce the interior weight below one, so a thin finger is not left half-transferred.

Eyes and teeth are restored to the source through the same smoothstep.
Lips and eyebrows keep a convex combination, $30\%$ transferred and $70\%$ original in the deployed configuration.
Those rules are inactive on the limb experiments, whose components lie below the face band.

\section{Mathematical Formalization}
\label{sec:formal}
Let $\varphi$ denote conversion to CIE Lab~\cite{cie2004}.
On the skin set, write the source and reference colours as the random vectors
\begin{equation}
X=(L,a,b)^{\top},
\qquad
Y=(L_Y,a_Y,b_Y)^{\top},
\end{equation}
with finite means $\mu_{\,\cdot\,}$ and positive standard deviations $\sigma_{\,\cdot\,}$.
A \emph{diagonal affine} map has no cross-talk between channels.
It is determined by a diagonal gain matrix and a translation:
\begin{equation}
\label{eq:diag}
f_{W,k}(x)=Wx+k,
\quad
W=\operatorname{diag}(w_L,w_a,w_b),
\quad
k\in\R^3.
\end{equation}
In coordinates,
\begin{equation}
\label{eq:matrix}
\begin{bmatrix} L'\\ a'\\ b' \end{bmatrix}
=
\begin{bmatrix}
w_L & 0 & 0\\
0 & w_a & 0\\
0 & 0 & w_b
\end{bmatrix}
\begin{bmatrix} L\\ a\\ b \end{bmatrix}
+
\begin{bmatrix} k_L\\ k_a\\ k_b \end{bmatrix}.
\end{equation}
The family has six real parameters.
The skin-restricted Reinhard transform is the member defined next.
The following section proves it is the only member compatible with the catalog constraints.

\subsection{Lightness}
The shading rule is: shift the lightness mean onto the reference, and do not rescale the spread.
With the population mean this is
\begin{equation}
\label{eq:L}
L'
=
L-\mu_{L}+\mu_{L,Y}
=
L+k_L,
\qquad
k_L=\mu_{L,Y}-\mu_{L}.
\end{equation}
Equation~\eqref{eq:L} is~\eqref{eq:matrix} with $w_L=1$.
No percentile trim appears in the gain because the gain is identically one.
The implementation estimates the two means by the trimmed functional of Section~\ref{sec:trim}, for the same robustness reason as on the chromatic axes.
Lemma~\ref{lem:trimm} shows that the trimmed mean is still matched exactly.

\subsection{Chroma}
Let $c\in\{a,b\}$.
The chromatic rule matches the first two moments, then limits the gain.
Write $\mu^{\star}$ and $\sigma^{\star}$ for the trimmed mean and trimmed standard deviation on the central $84\%$ of the channel, defined in~\eqref{eq:trim}.
The clamped gain and the matching translation are
\begin{equation}
\label{eq:gain}
\begin{aligned}
g_c
&=
\Proj_{[0.72,\,1.18]}
\Bigl(\frac{\sigma^{\star}_{c,Y}}{\sigma^{\star}_{c}}\Bigr),\\
c'
&=
g_c\bigl(c-\mu^{\star}_{c}\bigr)+\mu^{\star}_{c,Y}.
\end{aligned}
\end{equation}
Here $\Pi_{[p,q]}(\cdot)$ is Euclidean projection onto the interval, equivalently
\begin{equation}
\label{eq:clip}
g_c
=
\max\Bigl(0.72,\;
\min\bigl(1.18,\;\sigma^{\star}_{c,Y}/\sigma^{\star}_{c}\bigr)\Bigr).
\end{equation}
The translation form $k_c=\mu^{\star}_{c,Y}-g_c\mu^{\star}_{c}$ is the unique shift that keeps the trimmed mean on the reference once $g_c$ has been chosen.
Comparing~\eqref{eq:gain} with~\eqref{eq:reinhard}, the classical map is the special case $g_c=\sigma_{c,Y}/\sigma_{c}$ with the same formula applied to $L$ and with untrimmed moments.
The skin-restricted map refuses that special case on $L$ and refuses an unclamped gain on $a$ and $b$.

\subsection{Deployed blend}
Inside the matte the photograph is replaced at strength $s\in(0,1]$:
\begin{equation}
\label{eq:blend}
I'
=
(1-s\alpha)\,I
+
s\alpha\,\varphi^{-1}(L',a',b').
\end{equation}
The catalog configuration uses $s=0.7$.
Strength is a convex step toward~\eqref{eq:L}--\eqref{eq:gain}.
It is not an extra free parameter of the uniqueness theorem, which concerns the endpoint $s=1$.
If the Lab mean gap is at most $0.5$, below a just-noticeable CIE76 difference in the usual reading of that unit~\cite{cie2004,sharma2005}, the pipeline returns $I$ unchanged.
Table~\ref{tab:proc} lists every step.

\begin{table*}[t]
\caption{Skin-Restricted Reinhard Transfer}
\label{tab:proc}
\centering
\begin{tabular}{@{}c p{0.40\textwidth} c p{0.40\textwidth}@{}}
\toprule
Step & Operation & Step & Operation \\
\midrule
1 &
Segment $I$ with the 29-class body parser.
&
7 &
Sample the reference inside a $4\%$ border crop.
Drop pixels with lightness outside $(8,97)$ or chroma below $6$.
\\
2 &
Take face, torso, arms, hands, legs, and feet.
Remove dilated hair and eyeglasses.
&
8 &
Trimmed Lab moments on the central $84\%$ of that reference, and of the source on the mask eroded by $4$ pixels where $\alpha>0.62$.
\\
3 &
Reject limb components that fail the cloth test~\eqref{eq:cloth}.
&
9 &
If $\norm{\mu(S)-\mu(Y)}_2\le 0.5$, return $I$ unchanged.
\\
4 &
Add neck pixels only where the parser says background, in the box under the face, inside the Lab tolerance.
&
10 &
Lightness gain $w_L=1$.
Shift \mbox{$k_L=\mu^{\star}_{L,Y}-\mu^{\star}_{L}$}.
Clip $L'$ to $[0,100]$.
\\
5 &
Fill gaps that pass the local Lab test~\eqref{eq:local}.
Do not enter clothing or eyeglasses.
&
11 &
Project each chromatic gain onto $[0.72,1.18]$ by~\eqref{eq:clip} and match the trimmed $a$ and $b$ means by~\eqref{eq:gain}.
\\
6 &
Build $\alpha$: body interior $1$, cloth and eyeglasses $0$, smoothstep of width $3$ pixels on the outer shell.
Eyes and teeth stay at the source colour.
&
12 &
Mix lips and eyebrows as $0.3$ transferred and $0.7$ source.
Return the blend~\eqref{eq:blend}.
The catalog setting uses $s=0.7$.
\\
\bottomrule
\end{tabular}
\end{table*}

\section{Proof of Uniqueness}
\label{sec:proof}

\subsection{What is being optimised}
\begin{definition}[Catalog constraints]
\label{def:constraints}
A diagonal affine map~\eqref{eq:diag} is feasible when
\begin{enumerate}
\item \textbf{Orientation-preserving lightness.}
$w_L>0$ and $\Var(L')=\Var(L)$.
\item \textbf{Lightness location.}
$\E[L']=\mu_{L,Y}$.
\item \textbf{Chromatic location.}
For $c\in\{a,b\}$, $\E[c']=\mu_{c,Y}$.
\item \textbf{Chromatic scale, clamped.}
For each $c\in\{a,b\}$,
\[
w_c
=
\arg\min_{w\in[0.72,\,1.18]}
\bigl(w-\sigma_{c,Y}/\sigma_{c}\bigr)^{2}.
\]
\end{enumerate}
The population statement is written with $(\mu,\sigma)$.
Section~\ref{sec:trim} transfers it to $(\mu^{\star},\sigma^{\star})$ without changing the argument.
\end{definition}

The first constraint is the sentence ``do not rescale lightness.''
The third and fourth are ``match skin chroma'' inside the affine class, with a hard box on the gain.
Six conditions on six scalars make the system square.
They do not, by themselves, prove that a solution exists or that it is the right map; that is what the next three lemmas supply.

\subsection{Lightness gain and shift}
\begin{lemma}[Lightness parameters]
\label{lem:light}
Under Definition~\ref{def:constraints}, $w_L=1$ and $k_L=\mu_{L,Y}-\mu_{L}$.
Hence $L'=L-\mu_{L}+\mu_{L,Y}$, the spatial gradient satisfies $\nabla L'=\nabla L$ wherever the gamut clip is inactive, and the Pearson correlation of $L'$ with $L$ is $1$.
\end{lemma}

\begin{proof}
By~\eqref{eq:diag}, $L'=w_L L+k_L$, so
\begin{equation}
\label{eq:varscale}
\Var(L')=w_L^{2}\Var(L).
\end{equation}
The equality $\Var(L')=\Var(L)$ and $\Var(L)>0$ force $w_L^{2}=1$, hence $w_L\in\{+1,-1\}$.
The sign restriction $w_L>0$ discards $-1$, which would reverse highlights and shadows.
Thus $w_L=1$.
Taking expectations,
\begin{equation}
\E[L']=\E[L]+k_L=\mu_{L}+k_L.
\end{equation}
Condition 2 forces $k_L=\mu_{L,Y}-\mu_{L}$, which is~\eqref{eq:L}.
Differentiating a constant shift gives $\partial_i L'=\partial_i L$.
A strictly increasing affine image of a non-degenerate random variable has correlation $1$ with that variable.
\end{proof}

\begin{lemma}[Correlation is not shading]
\label{lem:corr}
If $L'=\rho L+k$ with $\rho>0$, the Pearson correlation of $L'$ with $L$ is $1$, but $\nabla L'=\rho\nabla L$ and $\sigma(L')=\rho\,\sigma(L)$.
\end{lemma}

\begin{proof}
The map is strictly increasing and affine, so the correlation equals $1$.
The chain rule multiplies every spatial derivative by $\rho$, and~\eqref{eq:varscale} gives the standard-deviation claim.
\end{proof}

Lemma~\ref{lem:corr} is why a correlation table cannot separate this transform from classical Reinhard.
Both are monotone in $L$.
Only $\rho$ decides whether a knuckle highlight survives.
Reinhard sets $\rho=\sigma_{L,Y}/\sigma_{L}$.
Lemma~\ref{lem:light} sets $\rho=1$.

\subsection{Chromatic match as one-dimensional transport}
For a single channel the quadratic Wasserstein distance between laws $P$ and $Q$ on $\R$ is
\begin{equation}
\label{eq:w2}
W_2^{2}(P,Q)
=
\inf_{\gamma\in\Pi(P,Q)}
\int \abs{u-v}^{2}\,\mathrm{d}\gamma(u,v),
\end{equation}
the infimum running over couplings with the correct marginals~\cite{villani2009}.
In one dimension the optimum is attained by the monotone rearrangement and is unique among nondecreasing maps.

\begin{lemma}[Gaussian rearrangement]
\label{lem:gauss}
Let $P=\mathcal{N}(\mu_{c},\sigma_{c}^{2})$ and $Q=\mathcal{N}(\mu_{c,Y},\sigma_{c,Y}^{2})$ with $\sigma_{c}>0$ and $\sigma_{c,Y}>0$.
The unique nondecreasing optimal map for~\eqref{eq:w2} is
\begin{equation}
\label{eq:w2map}
T_c(u)
=
\mu_{c,Y}+\frac{\sigma_{c,Y}}{\sigma_{c}}\bigl(u-\mu_{c}\bigr),
\end{equation}
and the transport cost is
\begin{equation}
\label{eq:w2cost}
W_2^{2}(P,Q)
=
\bigl(\mu_{c}-\mu_{c,Y}\bigr)^{2}
+
\bigl(\sigma_{c}-\sigma_{c,Y}\bigr)^{2}.
\end{equation}
\end{lemma}

\begin{proof}
Let $\Phi$ be the standard normal distribution function and $\Phi^{-1}$ its quantile function.
The distribution and quantile functions of $P$ and $Q$ are
\begin{equation}
\begin{aligned}
F_P(u)&=\Phi\bigl((u-\mu_{c})/\sigma_{c}\bigr),\\
F_Q^{-1}(p)&=\mu_{c,Y}+\sigma_{c,Y}\,\Phi^{-1}(p).
\end{aligned}
\end{equation}
The monotone rearrangement $T=F_Q^{-1}\circ F_P$ is therefore
\begin{equation}
\begin{aligned}
T(u)
&=
\mu_{c,Y}
+\sigma_{c,Y}\,
\Phi^{-1}\!\Bigl(\Phi\bigl((u-\mu_{c})/\sigma_{c}\bigr)\Bigr)\\
&=
\mu_{c,Y}+\frac{\sigma_{c,Y}}{\sigma_{c}}\bigl(u-\mu_{c}\bigr),
\end{aligned}
\end{equation}
which is~\eqref{eq:w2map}.
Uniqueness among nondecreasing maps is the one-dimensional Brenier property: any other nondecreasing coupling that pushes $P$ forward to $Q$ agrees with $T$ almost everywhere~\cite{villani2009}.
Substituting $U\sim P$ gives
$T(U)-\,U
=
(\mu_{c,Y}-\mu_{c})
+(\sigma_{c,Y}/\sigma_{c}-1)(U-\mu_{c})$.
The two summands are uncorrelated because the second has mean zero, so
\begin{equation}
\begin{aligned}
\E\bigl[(T(U)-U)^{2}\bigr]
&=
\bigl(\mu_{c,Y}-\mu_{c}\bigr)^{2}\\
&\quad+
\bigl(\sigma_{c,Y}/\sigma_{c}-1\bigr)^{2}\sigma_{c}^{2},
\end{aligned}
\end{equation}
which simplifies to~\eqref{eq:w2cost}.
\end{proof}

\begin{remark}[Why the affine restriction matters]
\label{rem:quantile}
For a non-Gaussian law the $W_2$-optimal map is still the quantile map $F_Y^{-1}\circ F$, and it is nonlinear.
That map reproduces the entire reference histogram, including its variance, which is histogram matching.
Lemma~\ref{lem:hist} records the conflict with shading.
Matching only the Gaussian with the same mean and variance, equivalently minimising the two-moment functional~\eqref{eq:w2cost}, is what keeps the chromatic step inside the diagonal affine family where the gain can be clamped and the lightness gain can be refused.
\end{remark}

\begin{lemma}[Histogram matching copies reference contrast]
\label{lem:hist}
Let $F$ and $F_Y$ be continuous and strictly increasing, and set $T=F_Y^{-1}\circ F$.
Then $T_{\#}P = Q$, so $\Var(T(L))=\Var(L_Y)$.
If those variances differ, no map that pushes the lightness law onto the reference law can satisfy Lemma~\ref{lem:light}.
\end{lemma}

\begin{proof}
For continuous strictly increasing $F$, $F(L)$ is uniform on $(0,1)$, and $F_Y^{-1}$ realises the law of $L_Y$.
Variance is a property of that law.
A translation has the variance of the source, so it realises the reference law only when the variances agree.
\end{proof}

The same obstruction applies to every transport that matches the lightness marginal, including iterative distribution transfer and any sliced step whose direction has a lightness component.

\subsection{The gain clamp is a projection}
Unclamped, Lemma~\ref{lem:gauss} sets $w_c=\sigma_{c,Y}/\sigma_{c}$.
A reference with a very small chromatic deviation would drive $w_c$ toward $0$ and wash the limb out.
A reference with a heavy chroma tail would drive $w_c$ above any reasonable saturation.
The catalog box forbids both.

\begin{lemma}[Projection of the gain]
\label{lem:proj}
Fix $w^{\circ}\in\R$ and $\gamma>0$.
On a nonempty closed interval $[p,q]$ the function $J(w)=\gamma(w-w^{\circ})^{2}$ has the unique minimiser $w^{\star}=\Proj_{[p,q]}(w^{\circ})$.
\end{lemma}

\begin{proof}
$J$ is strictly convex, so a closed interval contains exactly one minimiser.
If $w^{\circ}\in[p,q]$ the critical point $J'(w)=2\gamma(w-w^{\circ})=0$ lies in the interval and is the minimiser.
If $w^{\circ}<p$, then $J'(w)>0$ throughout $[p,q]$, so $J$ is increasing and the minimum is at $p$.
If $w^{\circ}>q$, $J$ is decreasing on the interval and the minimum is at $q$.
Those three cases are the definition of Euclidean projection onto $[p,q]$.
\end{proof}

\begin{lemma}[Shift after the clamp]
\label{lem:shift}
Fix $w_c$.
Among translations, $k_c=\mu_{c,Y}-w_c\mu_{c}$ is the unique choice with $\E[w_c\,c+k_c]=\mu_{c,Y}$.
In particular the mean match survives the clamp: it does not require $w_c=\sigma_{c,Y}/\sigma_{c}$.
\end{lemma}

\begin{proof}
$\E[w_c\,c+k_c]=w_c\mu_{c}+k_c$.
Setting this equal to $\mu_{c,Y}$ gives one linear equation for $k_c$.
\end{proof}

Thus clamping changes the second moment on purpose and does not reopen the first.
The residual variance gap is the quantity Lemma~\ref{lem:proj} refuses to close when the unclamped ratio lies outside $[0.72,1.18]$.

\subsection{Trimmed moments}
\label{sec:trim}
Highlights and crushed shadows should not set the skin tone.
For a real sample $U$ and $\alpha=0.08$, let $q_{\alpha}$ be the lower quantile and define the central window
\begin{equation}
\label{eq:trim}
U_{\tau}
=
\{\,u\in U: q_{\alpha}(U)\le u\le q_{1-\alpha}(U)\,\},
\end{equation}
which holds $84\%$ of the mass.
The trimmed moments are the ordinary mean and standard deviation of $U_{\tau}$, provided at least $32$ samples survive; otherwise the untrimmed moments are used.
An $8\%$ replacement in either tail cannot move the window, so the breakdown point of $\mu^{\star}$ is $\alpha=0.08$.

\begin{lemma}[Trimmed means commute with the map]
\label{lem:trimm}
Let $g>0$ and $T(u)=g(u-m)+n$.
Then
\begin{equation}
\mu^{\star}(T(U))
=
g\bigl(\mu^{\star}(U)-m\bigr)+n.
\end{equation}
In particular, if $m=\mu^{\star}(U)$ and $n=\mu^{\star}(Y)$, then $\mu^{\star}(T(U))=\mu^{\star}(Y)$.
A pure translation $T(u)=u+k$ also preserves $\sigma^{\star}$ and every trimmed variance.
\end{lemma}

\begin{proof}
Because $g>0$, $T$ is strictly increasing, so it sends quantiles to quantiles:
$q_{\beta}(T(U))=g(q_{\beta}(U)-m)+n$.
A sample lies in the central window of $U$ if and only if its image lies in the central window of $T(U)$.
The conditional mean therefore transforms by the same affine expression.
For a translation, $g=1$ and the window moves rigidly, so the centered deviations that enter $\sigma^{\star}$ are unchanged.
\end{proof}

Lemma~\ref{lem:trimm} is the reason the implementation may replace every population moment in~\eqref{eq:L} and~\eqref{eq:gain} by its trimmed counterpart and still meet Conditions 2 and 3 of Definition~\ref{def:constraints}, now read with $\mu^{\star}$.
The six-parameter lock is unchanged: it never depended on which functional produced the numbers $\mu$ and $\sigma$, only on those numbers being fixed before the gains and shifts are solved.

\subsection{The six parameters are determined}
\begin{theorem}[Uniqueness]
\label{thm:unique}
There is exactly one diagonal affine map satisfying Definition~\ref{def:constraints}.
It is the skin-restricted Reinhard transform
\begin{equation}
\label{eq:srrt}
\begin{aligned}
w_L&=1,
&
k_L&=\mu_{L,Y}-\mu_{L},\\
w_c&=\Proj_{[0.72,\,1.18]}(\sigma_{c,Y}/\sigma_{c}),
&
k_c&=\mu_{c,Y}-w_c\mu_{c},
\end{aligned}
\end{equation}
for $c\in\{a,b\}$.
Equivalently, in the matrix~\eqref{eq:matrix},
\begin{equation}
\label{eq:locked}
W=\operatorname{diag}(1,g_a,g_b),
\quad
k=
\begin{bmatrix}
\mu_{L,Y}-\mu_{L}\\
\mu_{a,Y}-g_a\mu_{a}\\
\mu_{b,Y}-g_b\mu_{b}
\end{bmatrix}.
\end{equation}
The same statement holds with every moment replaced by the trimmed moment of Section~\ref{sec:trim}.
\end{theorem}

\begin{proof}
Lemma~\ref{lem:light} determines $(w_L,k_L)$.
Lemma~\ref{lem:proj} determines $w_a$ and $w_b$, because each coordinate of $J$ depends on only one gain and the squared distance to the unclamped ratio is strictly convex.
Lemma~\ref{lem:shift} then determines $k_a$ and $k_b$.
A diagonal affine map has no remaining free parameter.
Lemma~\ref{lem:trimm} lets the identical algebra run on $(\mu^{\star},\sigma^{\star})$.
\end{proof}

Table~\ref{tab:lock} lists the lock.
The count shows only that the system is square.
Strict convexity of each one-dimensional projection, and the sign restriction $w_L>0$, show that it has one solution.
The content of Theorem~\ref{thm:unique} is which value is forced: $w_L=+1$, rather than the classical ratio $\sigma_{L,Y}/\sigma_{L}$ and rather than the orientation-reversing root $w_L=-1$.
Dense Monge maps, iterative distribution transfer, and neural recolouring lie outside the family and are not competitors for this statement.
Proposition~\ref{prop:jac} is the only step beyond the diagonal class, and it constrains only the lightness row.
Classical Reinhard is the same matrix family with the box removed and with $w_L$ set to $\sigma_{L,Y}/\sigma_{L}$.
Whenever that ratio is not $1$, it fails Condition 1 and is not a competitor for Theorem~\ref{thm:unique}.

\begin{table}[t]
\caption{The six parameters of~\eqref{eq:matrix} and the condition that fixes each of them.}
\label{tab:lock}
\centering
\renewcommand{\arraystretch}{1.15}
\begin{tabular}{@{}lll@{}}
\toprule
Parameter & Condition & Locked value \\
\midrule
$w_L$ & $\Var(L')=\Var(L)$, $w_L>0$ & $1$ \\
$k_L$ & $\E[L']=\mu_{L,Y}$ & $\mu_{L,Y}-\mu_{L}$ \\
$w_a$ & project $\sigma_{a,Y}/\sigma_{a}$ & $g_a\in[0.72,1.18]$ \\
$k_a$ & $\E[a']=\mu_{a,Y}$ & $\mu_{a,Y}-g_a\mu_{a}$ \\
$w_b$ & project $\sigma_{b,Y}/\sigma_{b}$ & $g_b\in[0.72,1.18]$ \\
$k_b$ & $\E[b']=\mu_{b,Y}$ & $\mu_{b,Y}-g_b\mu_{b}$ \\
\bottomrule
\end{tabular}
\end{table}

\subsection{Maps outside the diagonal family}
The catalog grade is a function of colour, not of pixel position.
Its spatial behaviour is completely described by the Jacobian of that function.

\begin{proposition}[No lightness cross-talk]
\label{prop:jac}
Suppose a differentiable colour map $f:\R^3\to\R^3$ satisfies $\nabla L'=\nabla L$ at every pixel of every smooth Lab image.
Then the lightness row of $Df$ is $(1,0,0)$ wherever the image attains an open set of colours.
In particular any affine map with that property has $w_L=1$ and zero off-diagonal entries in the lightness row, and Theorem~\ref{thm:unique} applies to its diagonal part once the chromatic constraints are imposed.
\end{proposition}

\begin{proof}
Let $x=(L,a,b)$ and let $f_L$ be the lightness coordinate of $f$.
Along an image $x(u,v)$,
\begin{equation}
\nabla f_L
=
Df_L(x)\;\nabla x
=
\partial_L f_L\,\nabla L
+\partial_a f_L\,\nabla a
+\partial_b f_L\,\nabla b.
\end{equation}
The hypothesis says this equals $\nabla L$ for every smooth image, hence for every choice of the three spatial gradients at a point.
Taking an image whose colour gradient spans $\R^3$ in a neighbourhood forces $\partial_L f_L=1$ and $\partial_a f_L=\partial_b f_L=0$.
An affine lightness coordinate is $w_{LL}L+w_{La}a+w_{Lb}b+k_L$, so the same identities are $w_{LL}=1$ and $w_{La}=w_{Lb}=0$.
\end{proof}

Proposition~\ref{prop:jac} is the reason a full $3\times 3$ linear map is not a way around the theorem.
The Monge map of two colour Gaussians is the unique quadratic-cost coupling in $\R^3$~\cite{pitie2007cvmp,villani2009}, and its matrix is generally dense.
Density means a change in $a$ moves $L$.
The quadratic colour cost does not contain $\norm{\nabla L'-\nabla L}$, so the coupling uses a smaller lightness variance in the reference to reduce transport cost.
That is optimal for a different problem.
It is infeasible for Definition~\ref{def:constraints}.

\begin{corollary}[Operator norm]
\label{cor:lip}
Under~\eqref{eq:locked}, the Euclidean operator norm of $W$ is
\begin{equation}
\norm{W}_2
=
\max(1,g_a,g_b)
\le 1.18.
\end{equation}
Hence $f_{W,k}$ is $1.18$-Lipschitz.
Colour differences on skin cannot be amplified by more than this constant, regardless of how small $\sigma^{\star}_{c}$ becomes.
\end{corollary}

\begin{proof}
A diagonal matrix has operator norm equal to its largest absolute diagonal entry, and each $g_c$ lies in $[0.72,1.18]$ by construction.
\end{proof}

Classical Reinhard has no such bound.
If a shadow region drives $\sigma_{c}$ toward zero, the unclamped ratio $\sigma_{c,Y}/\sigma_{c}$ diverges and the map ceases to be uniformly Lipschitz in the source statistics.

\subsection{Infeasibility of lightness rescaling}
\begin{theorem}[Infeasible classical maps]
\label{thm:infeasible}
Assume $\sigma_{L,Y}\ne\sigma_{L}$.
Then classical Reinhard~\eqref{eq:reinhard}, the monotone lightness rearrangement of Lemma~\ref{lem:hist}, and every affine map whose lightness row is not the translation of Lemma~\ref{lem:light} are infeasible for Definition~\ref{def:constraints}.
No infeasible map is a better solution of that definition, whatever its chromatic error.
\end{theorem}

\begin{proof}
Lemma~\ref{lem:corr} gives $\sigma(L')=\rho\,\sigma(L)$ for an affine lightness gain $\rho$.
Condition 1 forces $\rho=1$.
Reinhard uses $\rho=\sigma_{L,Y}/\sigma_{L}\ne 1$.
Lemma~\ref{lem:hist} excludes every exact push-forward onto a lightness law of a different variance.
Proposition~\ref{prop:jac} forces every admissible affine lightness row to be a translation.
Theorem~\ref{thm:unique} already names the unique feasible point.
\end{proof}

\subsection{Strength and support sit outside the six parameters}
\begin{proposition}[Partial strength]
\label{prop:strength}
On a region where $\alpha=1$, the blend $I'=(1-s)I+s\,f(I)$ has mean $(1-s)\mu+(s)\mu_f$.
The residual mean error is $(1-s)$ times the gap that $f$ closed.
If $f$ acts on lightness by a translation, then $\nabla L'=\nabla L$ still holds.
\end{proposition}

\begin{proof}
Expectation is linear, so $\mu-\mu_{I'}=(1-s)(\mu-\mu_f)$.
A convex combination of $L$ with $L+k_L$ equals $L+sk_L$, whose spatial gradient equals $\nabla L$.
\end{proof}

The deployed choice $s=0.7$ therefore leaves $30\%$ of the mean gap on purpose.
Scoring that residual as a failure of Theorem~\ref{thm:unique} confuses the endpoint map with a later convex step.

\begin{proposition}[Support]
\label{prop:support}
If $\alpha=0$ on a set $\Omega_0$, then~\eqref{eq:blend} satisfies $I'=I$ on $\Omega_0$.
Clothing labels are assigned $\alpha=0$ after every growth step, so the identity on cloth is exact up to the three-pixel smoothstep ring and compression.
\end{proposition}

\begin{proof}
Substitute $\alpha=0$ into~\eqref{eq:blend}.
\end{proof}

\section{Experiments}
\label{sec:exp}

\subsection{Protocol}
The catalog contains eleven photographs.
Nine of them contain at least one skin component whose top lies below $32\%$ of the image height and which is not a tall region starting in the upper body.
Those nine are the limb set: hands, arms, legs, and feet.
The other two contain only face-band skin under this rule and do not enter the limb tables.
This is the whole image set.
It is not a public benchmark and it is not a hundred-image study.

Three reference textures supply the target tones.
Their trimmed Lab means are tone~A, warm, $(L,a,b)=(60.93,\,19.73,\,34.10)$; tone~B, light, $(77.57,\,14.36,\,19.02)$; and tone~C, deep, $(39.58,\,19.30,\,26.42)$.
Crossed with the nine photographs this is $27$ pairs.

Every method is applied to the same limb pixels.
Statistics are estimated on those pixels and on the same reference sample, so the parser is not a confound.
A second measurement applies classical Reinhard to every pixel of the frame, which is how that method was defined~\cite{reinhard2001}, and records the mean absolute change on non-skin pixels in a $31$-pixel band around the limbs.

The baselines are~\eqref{eq:reinhard}, the linear Monge map~\cite{pitie2007cvmp}, per-channel Lab histogram matching~\cite{neumann2005}, an HSV mean shift with the value axis scaled by $0.35$, and iterative distribution transfer with $12$ random rotations and $128$ bins~\cite{pitie2007cviu}.
Twelve rotations are a finite approximation of a method that converges as the rotation count grows.
The proposed column is~\eqref{eq:srrt} at strength $s=1$, using trimmed moments.

Shading is the lightness standard-deviation ratio $\sigma(L')/\sigma(L)$ and the ratio of mean Sobel energies of $L$ inside the mask.
Both equal $1$ when Lemma~\ref{lem:light} holds and clipping is rare.
Chromatic error $\Delta C_{ab}$ is the Euclidean gap of trimmed $(a,b)$ means~\cite{cie2004}.
A scalarization of the two columns is
\begin{equation}
\label{eq:Jexp}
J_{\beta}
=
\Delta C_{ab}
+
\beta\,\abs{\frac{\sigma(L')}{\sigma(L)}-1},
\end{equation}
computed per pair and then averaged.
The weight $\beta$ is a preference.
The value $\beta=10$ prices a $10\%$ contrast error as one CIE Lab unit of chroma.
It is not an independent measurement, and a method that rescales lightness by construction pays it automatically.

\subsection{Quantitative comparison}
Table~\ref{tab:main} and Fig.~\ref{fig:metrics} summarise the $27$ pairs.
Histogram matching wins raw $\Delta C_{ab}$ because, by Lemma~\ref{lem:hist}, it copies the reference law.
It also copies the reference contrast: the lightness standard deviation falls to $0.48$ of the photograph, and on tone~B it falls to $0.31\pm0.11$.
Reinhard and the Monge map do the same through a single gain.
Their Pearson correlations with source lightness are $0.9997$ and $0.9976$.
Lemma~\ref{lem:corr} predicts this pattern exactly.
A correlation column would have hidden the defect.

\begin{table*}[t]
\caption{Limb skin, $27$ photograph--reference pairs.
$\Delta C_{ab}$ is the trimmed chromatic error (lower is closer in hue).
The $\sigma$ and $\nabla$ ratios equal $1$ when shading is unchanged.
$J_{10}$ is the scalarization~\eqref{eq:Jexp} at weight $10$.
Table~\ref{tab:beta} shows how the ranking moves when that weight changes.}
\label{tab:main}
\centering
\begin{tabular}{lcccc}
\toprule
Method & $\Delta C_{ab}$ & $\sigma(L')/\sigma(L)$ & $\nabla$ ratio & $J$ \\
\midrule
Reinhard~\cite{reinhard2001}
  & $0.21\pm0.09$ & $0.48\pm0.22$ & $0.65\pm0.25$ & $5.44\pm2.21$ \\
Monge--Kantorovich~\cite{pitie2007cvmp}
  & $0.21\pm0.10$ & $0.48\pm0.22$ & $0.66\pm0.26$ & $5.39\pm2.25$ \\
Histogram~\cite{neumann2005}
  & $0.02\pm0.01$ & $0.48\pm0.22$ & $0.68\pm0.25$ & $5.22\pm2.24$ \\
HSV shift
  & $6.25\pm3.61$ & $0.95\pm0.06$ & $0.99\pm0.07$ & $6.91\pm3.90$ \\
Iterative distribution transfer~\cite{pitie2007cviu}
  & $3.90\pm2.28$ & $0.65\pm0.14$ & $0.77\pm0.16$ & $7.39\pm2.41$ \\
Skin-restricted Reinhard~\eqref{eq:srrt}
  & $0.77\pm1.23$ & $0.97\pm0.03$ & $1.06\pm0.11$ & $1.03\pm1.47$ \\
\bottomrule
\end{tabular}
\end{table*}

The proposed map pays $0.77$ Lab units of chromatic residual.
That residual is the gamut clip and the Lab encode and decode.
It is largest on tone~B ($1.45\pm1.73$), the tone whose unclamped Reinhard lightness gain is the most destructive ($0.31$).
Lightness contrast stays at $0.974\pm0.029$ on all three tones.
The HSV shift keeps contrast and misses chroma by $6.25$ units, because a shift of HSV means is not a Lab moment match.
Iterative distribution transfer at $12$ rotations reduces contrast to $0.65$ and still leaves $3.90$ units of trimmed chromatic error.
The primary comparison is the pair of columns $\Delta C_{ab}$ and $\sigma(L')/\sigma(L)$.
Reinhard, the Monge map, and histogram matching win the first column by losing the second.

Table~\ref{tab:beta} is the same $27$ pair records scored at three weights.
At $\beta=0$ the score is $\Delta C_{ab}$ and histogram matching is first.
At $\beta=2$ and at $\beta=10$ the skin-restricted map is first, because the three lightness-rescaling methods pay for the contrast they remove.
A large weight widens a gap that the contrast column already shows.
It does not establish a ranking on its own.

\begin{table}[b]
\caption{Scalarization $J_{\beta}$ on the same $27$ limb pairs.
At $\beta=0$ the score is chromatic error alone.}
\label{tab:beta}
\centering
\begin{tabular}{lccc}
\toprule
Method & $\beta=0$ & $\beta=2$ & $\beta=10$ \\
\midrule
Reinhard & $0.21\pm0.09$ & $1.25\pm0.46$ & $5.44\pm2.21$ \\
Monge & $0.21\pm0.10$ & $1.24\pm0.47$ & $5.39\pm2.25$ \\
Histogram & $0.02\pm0.01$ & $1.06\pm0.45$ & $5.22\pm2.24$ \\
HSV shift & $6.25\pm3.61$ & $6.38\pm3.66$ & $6.91\pm3.90$ \\
IDT & $3.90\pm2.28$ & $4.59\pm2.24$ & $7.39\pm2.41$ \\
Proposed & $0.77\pm1.23$ & $0.82\pm1.27$ & $1.03\pm1.47$ \\
\bottomrule
\end{tabular}
\end{table}

\begin{figure*}[t]
\centering
\includegraphics[width=\textwidth]{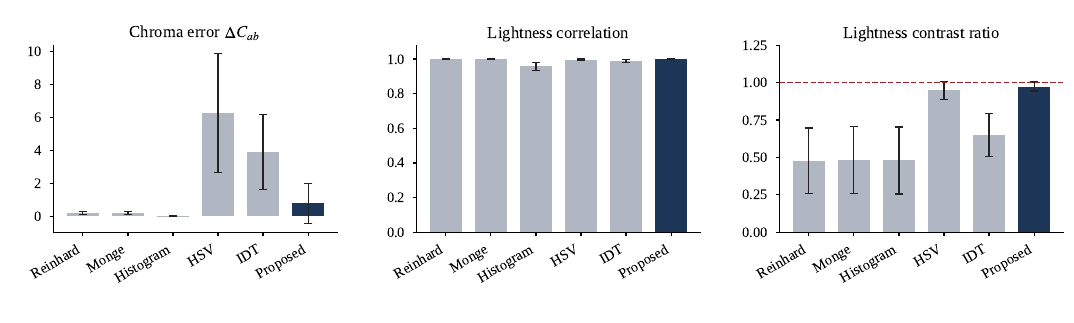}
\caption{Means and standard deviations over $27$ limb pairs.
Left: chromatic error.
Centre: correlation of lightness with the source.
It stays high for every affine method and therefore does not measure shading (Lemma~\ref{lem:corr}).
Right: lightness contrast ratio.
The dashed line is perfect shading preservation.
Reinhard, the Monge map, and histogram matching sit near $0.5$.
The skin-restricted map sits on the dashed line.}
\label{fig:metrics}
\end{figure*}

The deployed pipeline uses the same map at strength $s=0.7$.
On the same limb pixels its contrast ratio is $0.986\pm0.024$ and its lightness correlation is $0.997$.
Its chromatic error rises to $4.82\pm2.42$, the residual $(1-s)$ gap of Proposition~\ref{prop:strength}.
Off the mask, the mean absolute change is $0.69\pm0.37$ grey levels.
Global Reinhard, in a band around the limbs, changes non-skin pixels by $57.5\pm16.6$ grey levels.
Proposition~\ref{prop:support} is the reason.

\subsection{Sensitivity of the constants the theorem does not fix}
\label{sec:sens}
Theorem~\ref{thm:unique} holds for every nonempty closed gain interval and every trim window.
The endpoints $[0.72,1.18]$ and the central $84\%$ (trim fraction $\alpha=0.08$, percentiles $8$ and $92$) are catalog operating points.
Table~\ref{tab:sens} varies them on the same $27$ limb pairs, with the other control held at the deployed value.

The chromatic mean match does not depend on the gain, so widening the box barely moves $\Delta C_{ab}$: $0.74\pm1.20$ with no clamp, $0.77\pm1.23$ on the deployed box, and $0.80\pm1.28$ on $[0.90,1.10]$.
The lightness ratio stays at $0.97$ throughout, because the lightness gain is locked at $1$.
What the box changes is chroma contrast.
Of the $54$ chromatic gains ($27$ pairs, two channels), $42$ fall outside $[0.72,1.18]$.
Unclamped, the mean $b$-channel standard-deviation ratio falls to $0.45$, and the smallest gain on the set is $0.15$.
The deployed box holds the mean $b$ ratio at $0.72$ and the mean $a$ ratio at $0.82$.
The tighter box $[0.90,1.10]$ preserves still more of the source chroma contrast.
The interval is a preference between source chroma contrast and the reference scale, and it is the Lipschitz bound of Corollary~\ref{cor:lip}.
It is not a fitted optimum.

The trim does not move the lightness ratio, as Lemma~\ref{lem:trimm} predicts for a unit lightness gain.
Among the four windows, the central $84\%$ has the lowest chromatic residual: $0.77\pm1.23$, against $0.92\pm1.18$ with no trim and $0.86\pm1.19$ when $15\%$ is removed from each tail.

\begin{table*}[t]
\caption{Operating-point sensitivity on the $27$ limb pairs.
Left: chromatic gain interval, trim fixed at the central $84\%$.
The $a$ and $b$ columns are mean ratios of channel standard deviation to the source.
Right: trim fraction $\alpha$, gain interval fixed at $[0.72,1.18]$.
A window $\alpha$ keeps the central $1-2\alpha$ of each channel.}
\label{tab:sens}
\centering
\begin{tabular}{lcccc}
\toprule
Gain interval & $\Delta C_{ab}$ & $\sigma(L')/\sigma(L)$ & $a$ ratio & $b$ ratio \\
\midrule
Unclamped & $0.74\pm1.20$ & $0.973\pm0.030$ & $0.73$ & $0.45$ \\
$[0.50,1.50]$ & $0.75\pm1.21$ & $0.973\pm0.029$ & $0.74$ & $0.56$ \\
$[0.72,1.18]$ & $0.77\pm1.23$ & $0.974\pm0.029$ & $0.82$ & $0.72$ \\
$[0.90,1.10]$ & $0.80\pm1.28$ & $0.974\pm0.029$ & $0.94$ & $0.90$ \\
\bottomrule
\end{tabular}
\hspace{1.5em}
\begin{tabular}{lcc}
\toprule
Trim $\alpha$ & $\Delta C_{ab}$ & $\sigma(L')/\sigma(L)$ \\
\midrule
$0$ (no trim) & $0.92\pm1.18$ & $0.973\pm0.030$ \\
$0.05$ & $0.79\pm1.23$ & $0.973\pm0.029$ \\
$0.08$ & $0.77\pm1.23$ & $0.974\pm0.029$ \\
$0.15$ & $0.86\pm1.19$ & $0.974\pm0.030$ \\
\bottomrule
\end{tabular}
\end{table*}

\subsection{Face band, without photographs}
The limb rule excludes every component whose centroid lies in the top $32\%$ of the image.
Those components are the face and neck band.
They were scored with the same Lab statistics and were not saved as images.
All eleven photographs contain such a component, hence $33$ pairs.
On that band the skin-restricted map has $\Delta C_{ab}=0.60\pm0.89$ and lightness-contrast ratio $0.97\pm0.04$.
Classical Reinhard has $0.23\pm0.08$ and $0.41\pm0.17$.
The split is the same as on the limbs: a smaller chromatic residual bought by cutting contrast to under half.
These numbers are mask-restricted moments.
They are not a study of eyes, lips, hairlines, or identity.
Those catalog faces are not displayed.
Appendix~\ref{app:visual} shows one other portrait, and it is not part of this measurement.

\subsection{Limb images}
Fig.~\ref{fig:hand} is the clasped hands and forearms of one catalog photograph, as the source and as references 1--14.
The sleeve and the garment stay outside the transfer.
Fig.~\ref{fig:leg} is a lower leg and foot.
On that foot, Reinhard, the Monge map, and histogram matching are run on tone~B, the light reference and the hardest contrast case in Table~\ref{tab:main}.
They erase the gradient along the shin.
The last three columns are~\eqref{eq:srrt} on tones A, B, and C, so tone~B is the same reference the classical columns use.
The same shading field is visible under all three pigments: by Lemma~\ref{lem:light} the lightness pattern is not a function of the reference variance.
Every column of Fig.~\ref{fig:leg} is masked, so the sandal does not move.

\begin{figure*}[t]
\centering
\includegraphics[width=\textwidth]{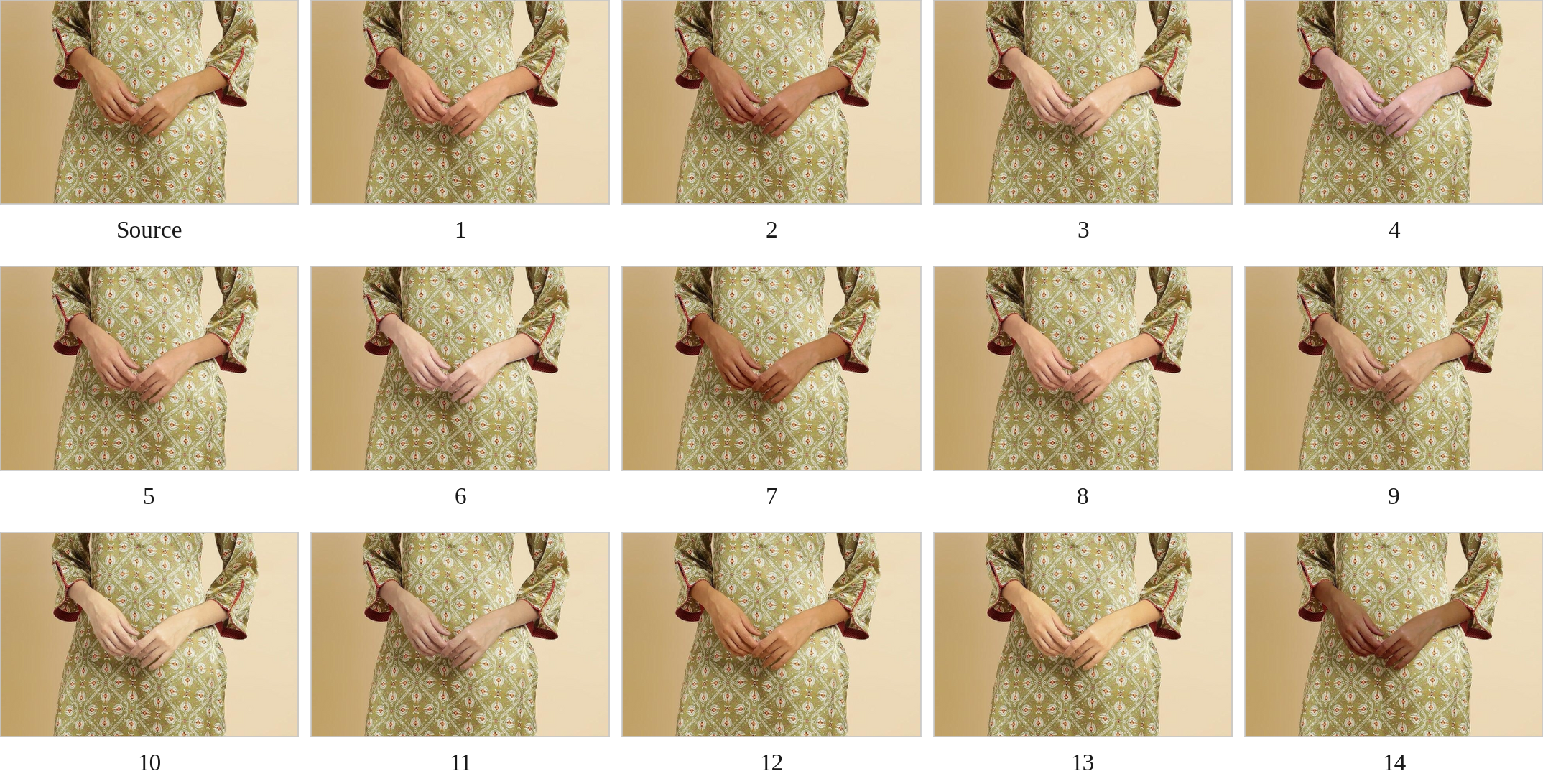}
\caption{Clasped hands and forearms, cropped from one catalog photograph.
Left to right, top to bottom: the source, then references 1--14.
The sleeve and the garment do not move.
The face is not shown.}
\label{fig:hand}
\end{figure*}

\begin{figure*}[t]
\centering
\includegraphics[width=\textwidth]{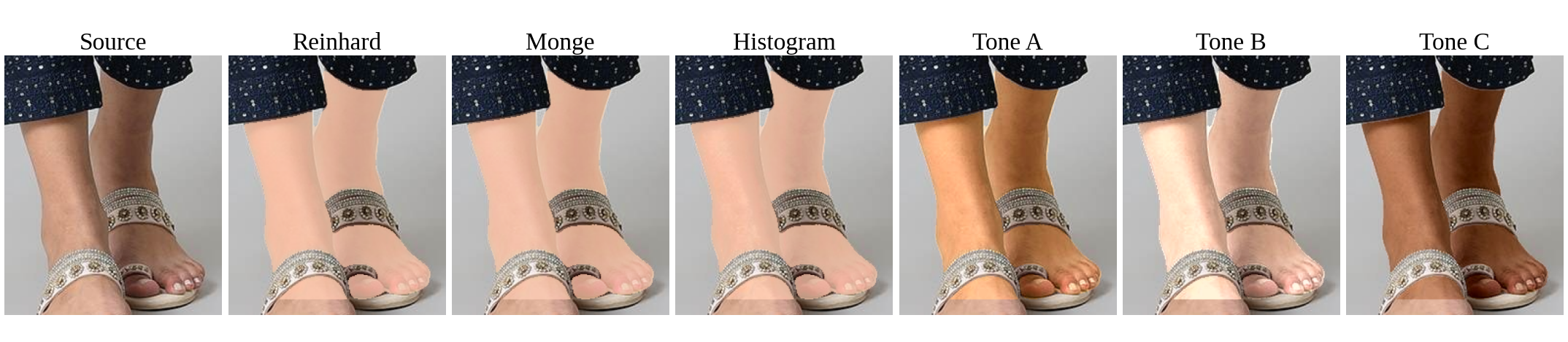}
\caption{Lower leg and foot.
Left to right: source; Reinhard, Monge, and histogram matching on tone~B; then the skin-restricted map on tones~A, B, and~C.
The shin gradient and the ankle shadow remain for tones~A--C and collapse for Reinhard, Monge, and histogram matching.}
\label{fig:leg}
\end{figure*}

\subsection{Scope}
The image set is eleven catalog photographs and three reference tones: $27$ limb pairs and $33$ face-band pairs.
There is no held-out public benchmark, no hundred-image collection, and no preference study.
The catalog faces are reported only as the mask statistics above.
Appendix~\ref{app:visual} is one further portrait, qualitative, and not part of either table.
The cloth thresholds in~\eqref{eq:cloth} and the local-fill weights in~\eqref{eq:local} are not ablated; there is no labeled set on which a false-positive rate could be computed.
Theorem~\ref{thm:infeasible} does not need a larger set to place lightness-rescaling maps outside the feasible set.
It does not say that no future map could also satisfy the constraint.
A gradient-preserving optimisation in the sense of Xiao and Ma~\cite{xiao2009}, restricted to the same diagonal family, would reach the same lightness translation, which here is closed form.
The optional face stage, a face parser~\cite{yu2018,lee2020maskgan}, a restoration network~\cite{wang2021gfpgan}, and a Laplacian texture split~\cite{burt1983,tomasi1998}, is not part of either table.

\section{Conclusion}
Inside the diagonal affine family, catalog skin transfer is a six-parameter problem with a hard constraint on lightness.
The shading rule forces the lightness gain to $+1$ and the lightness shift to the difference of means.
One-dimensional quadratic transport on each chromatic axis, projected onto a gain interval, fixes the other four.
The skin-restricted Reinhard transform is the solution of that square system.
The interval $[0.72,1.18]$ and the central-$84\%$ trim are operating points of the solution, not consequences of the counting argument: the trim window with the lowest chromatic residual on this set is the one in use, and the gain box is what stops chroma contrast from falling to $0.45$.
The trimmed estimator inherits the same parameter lock because positive affine maps commute with quantile windows.
Classical moment matching, histogram matching, and Gaussian Monge maps remain optimal for the unconstrained colour problem and are infeasible for this one whenever the reference lightness variance differs from the photograph.
On the $27$ limb pairs it is the only tested map that holds lightness contrast within $3\%$ of the photograph while keeping chromatic error under one CIE Lab unit.
The same contrast split appears on the face band of the catalog set, which is reported as mask statistics.
Appendix~\ref{app:visual} is qualitative and separate.
A scalarization that charges contrast error ranks this map first only once that charge is turned on; with the charge off, histogram matching wins on chroma alone.

\clearpage
\onecolumn
\appendices
\section{Qualitative Comparison across Fourteen References}
\label{app:visual}
One hand and one portrait were each transferred to fourteen references, numbered in the order of that set.
Every panel inside a plate is the same crop.
The plates are qualitative.
They are not part of the twenty-seven limb pairs in Table~\ref{tab:main}.

\begin{figure}[!ht]
\centering
\includegraphics[width=\textwidth]{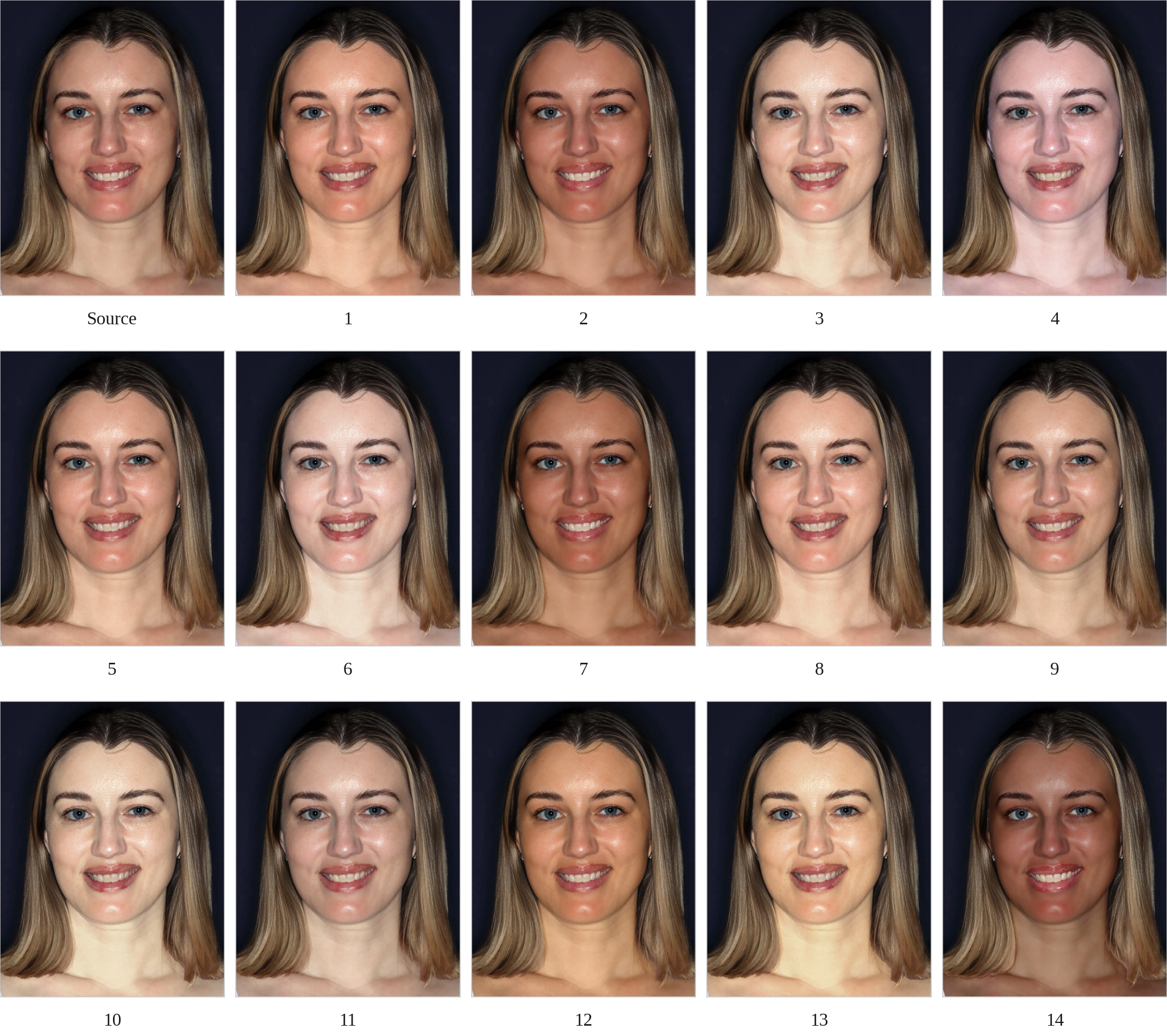}
\caption{Portrait across fourteen references.
The first panel is the source.
Panels 1--14 follow the appendix reference order.
Hair, eyes, and teeth remain at the source colour.
This plate is not part of Table~\ref{tab:main}.}
\label{fig:app-face}
\end{figure}

\clearpage
\begin{figure}[!ht]
\centering
\includegraphics[width=\textwidth]{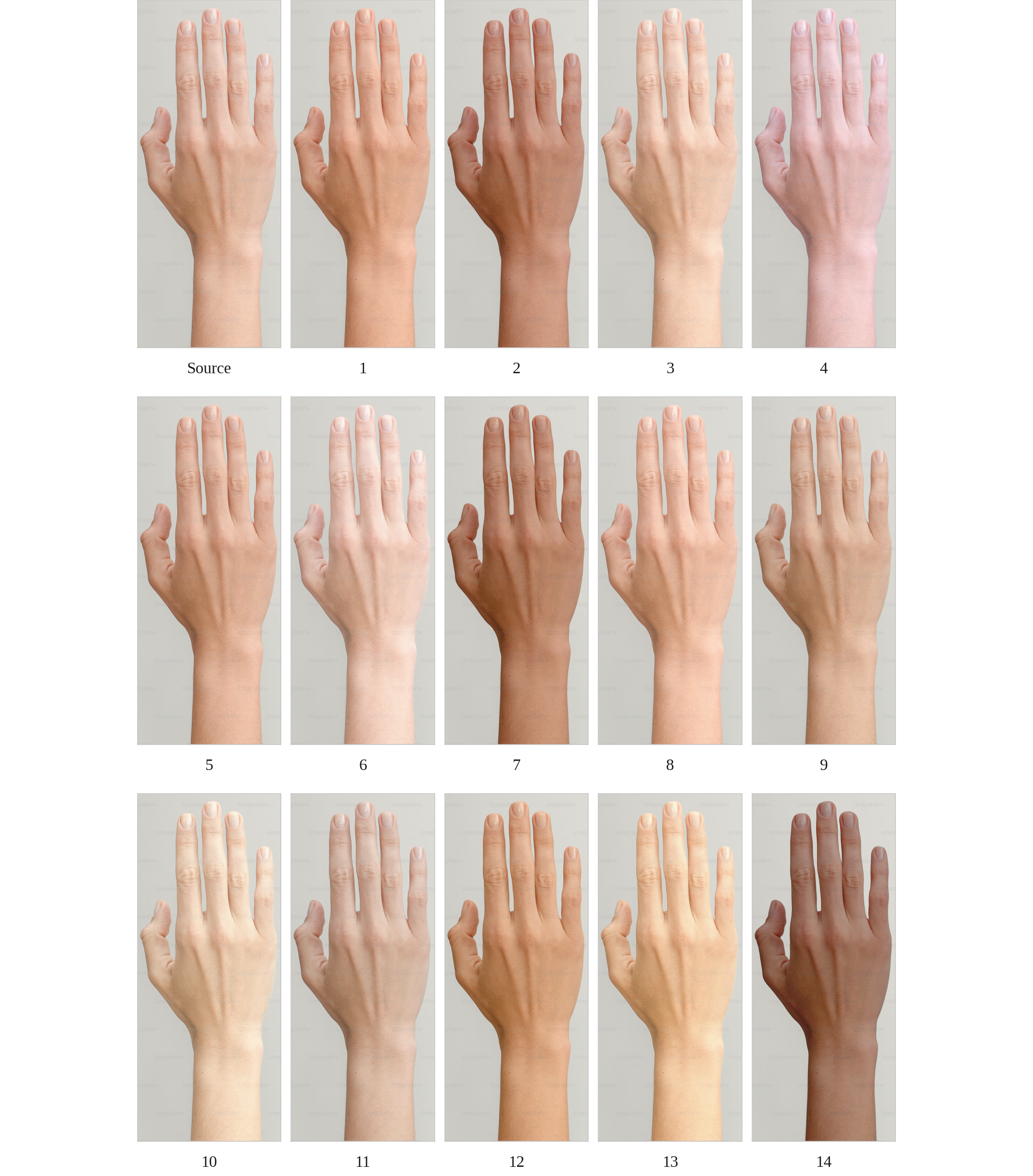}
\caption{Hand across the same fourteen references, in the same order as Fig.~\ref{fig:app-face}.
Knuckle creases and the wrist shadow stay in place; the pigment changes.}
\label{fig:app-hand}
\end{figure}

\end{document}